\documentclass{article}

\usepackage{PRIMEarxiv}

\usepackage[utf8]{inputenc} 
\usepackage[T1]{fontenc}    
\usepackage{hyperref}       
\usepackage{url}            
\usepackage{booktabs}       
\usepackage{amsfonts}       
\usepackage{nicefrac}       
\usepackage{microtype}      
\usepackage{lipsum}
\usepackage{fancyhdr}       
\usepackage{graphicx}       
\graphicspath{{media/}}     
\usepackage{amsmath,times,algorithm,algpseudocode}
\usepackage{tikz}
\usetikzlibrary[topaths]
\usepackage{pgfplots}
\usepackage{amsthm}

\newtheorem{definition}{Definition}
\newtheorem{example}{Example}
\newtheorem{proposition}{Proposition}

\newtheorem{problem}{Problem}

\newcommand{\AF}{\mathcal{F}}
\newcommand{\A}{\mathcal{A}}

\newcommand{\D}{\mathcal{D}}
\newcommand{\Dec}{\mathcal{DEC}}
\newcommand{\Ssp}{\mathcal{SSP}}
\newcommand{\kSsp}{k\mathcal{SSP}}

\newcommand{\Att}{\mathtt{Att}}

\title{Inferring Attack Relations for Gradual Semantics
}

\author{
  Nir Oren\\
  University of Aberdeen \\
  Scotland\\
  \texttt{n.oren@abdn.ac.uk} \\
   \And
  Bruno Yun \\
  University of Aberdeen\\
  Scotland\\
  \texttt{bruno.yun@abdn.ac.uk} \\
}

\begin{document}
\maketitle

\begin{abstract}
A gradual semantics takes a weighted argumentation framework as input and outputs a final acceptability degree for each argument, with different semantics performing the computation in different manners. In this work, we consider the problem of attack inference. That is, given a gradual semantics, a set of arguments with associated initial weights, and the final desirable acceptability degrees associated with each argument, we seek to determine whether there is a set of attacks on those arguments such that we can obtain these acceptability degrees. The main contribution of our work is to demonstrate that the associated decision problem, i.e., whether a set of attacks can exist which allows the final acceptability degrees to occur for given initial weights, is NP-complete for the weighted h-categoriser and cardinality-based semantics, and is polynomial for the weighted max-based semantics, even for the complete version of the problem (where all initial weights and final acceptability degrees are known). We then briefly discuss how this decision problem can be modified to find the attacks themselves and conclude by examining the partial problem where not all initial weights or final acceptability degrees may be known.
\end{abstract}


\section{Introduction}
Abstract argumentation semantics associate a justification status with a set of arguments based on interactions between arguments. Such interactions can include inter-argument attacks \cite{dung_acceptability_1995}, or preference-based defeats \cite{modgil_general_2013,kaci_working_2011,yun_arguing_2016, amgoud_two_2011,amgoud_rich_2014}, or may consider both of these together with some supportive relationship \cite{amgoud_bipolarity_2008,amgoud_weighted_2018,cayrol_acceptability_2005,mossakowski_modular_2018}. Given such a framework, an argument may be considered justified if it appears in one (resp. many or all) \emph{extensions}, where such extensions -- sets of non-conflicting arguments accepted together -- are computed according to the argumentation semantics, or based on the label assigned to the argument (we refer the reader to \cite{baroni_introduction_2011} for an introduction to 
 the most common argumentation semantics).

Arguments often have associated weights, and different semantics have been proposed which consider the weight associated with an argument \cite{coste-marquis_selecting_2012,amgoud_weighted_2018,mossakowski_modular_2018,inverse_problem}. While some semantics developed in this context return whether an argument is, or is not justified \cite{coste-marquis_selecting_2012}, many other \emph{ranking-based} semantics either return a preference ordering over arguments by making use of the argument's initial weight or associate a numeric \emph{final acceptability degree} with each argument in the framework. Such ranking-based semantics have been used to --- for example --- identify irrationality in reasoning \cite{irwin22forecasting} by examining whether the initial weights associated with an argument affect the argument's final acceptability degree in an appropriate manner (i.e., consistent with the ranking-based semantics being used); refining the results obtained by extension semantics \cite{bonzon_combining_2018,yun_viewpoints_2018}; and applied to multi-agent settings \cite{DBLP:conf/atal/TarleBM22}.

The top portion of Fig. \ref{fig:process} shows the reasoning process used when reasoning using gradual semantics. Here, a weighted argumentation framework consisting of arguments, attacks and initial weights associated with arguments is provided as input. A gradual semantics is then used to compute a \emph{final acceptability degree} for each argument. In many semantics, these final acceptability degrees are then used to compute a preference ordering over the arguments.

More recent work has considered different combinations of inputs and outputs to the problem. For example, \cite{inverse_problem} seeks to identify a set of initial weights for arguments based on the final argument preference ordering and argumentation semantics, while \cite{MahesarOV20} determines the preferences between arguments given argument justification status, semantics and argumentation framework. In this paper, and as illustrated in the bottom portion of Fig. \ref{fig:process}, we consider the problems of determining whether a set of attacks between arguments can be identified given specific argumentation semantics, the final acceptability degrees, and the initial weights. As discussed further in Section \ref{sec:Discussion}, we leave the problem of using preferences over arguments as input for our problem as future work. While it is true that the problem we consider is somewhat abstract and has limited real-world applications, it serves as a departure point for potentially important applications of argumentation to opponent modelling \cite{reinstra13opponent,POLBERG2018487} and preference elicitation \cite{MahesarOV20}.

\begin{figure}
    \centering
    \includegraphics[width=0.7\textwidth]{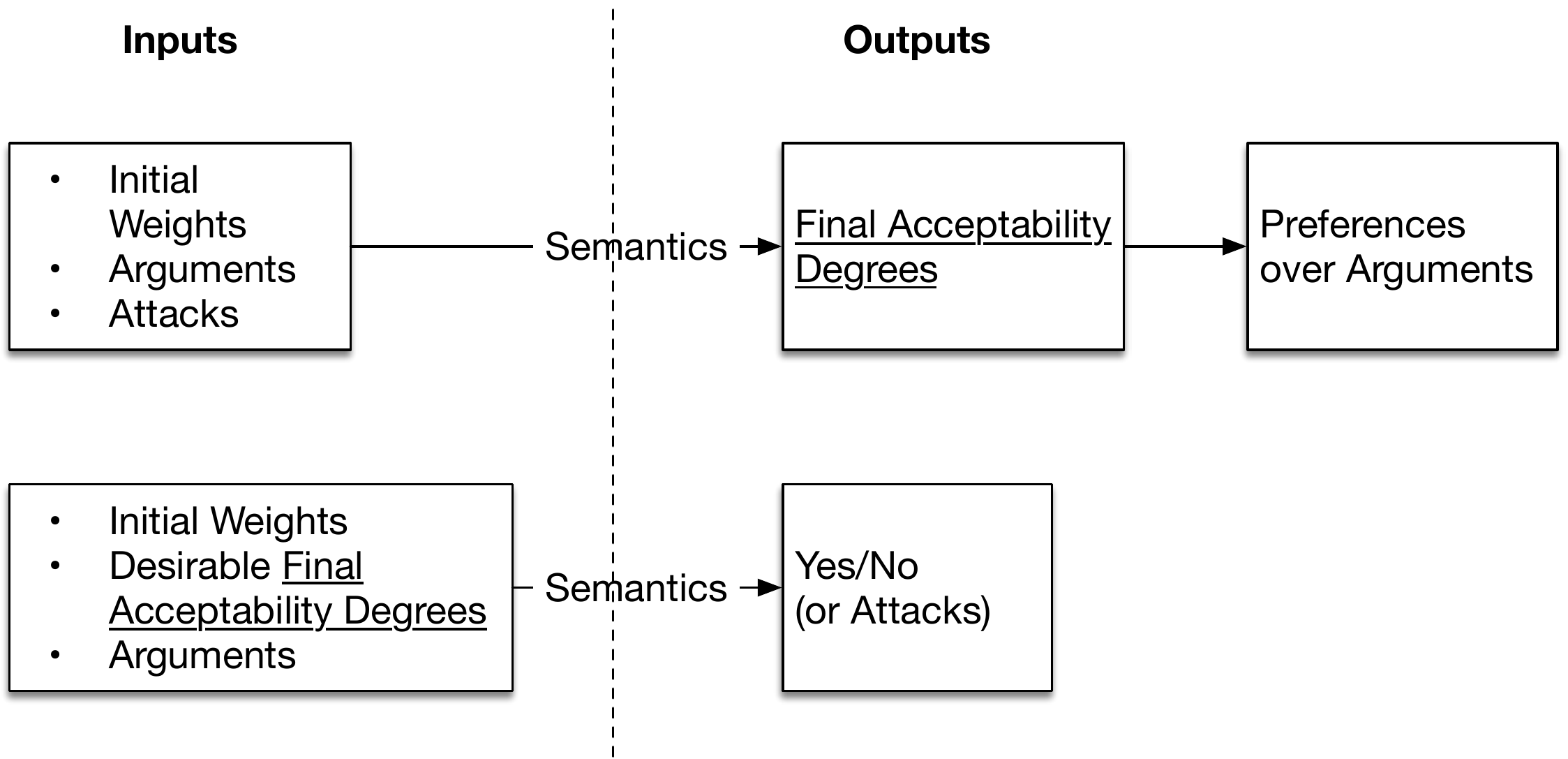}
    \caption{The process (top) by which a gradual semantics is applied to compute an acceptability preference between arguments and (bottom) the inverse problem considered in this paper. The ``final acceptability degrees" are now used as input. That is, given arguments, initial weights, and desirable final acceptability degrees, can we find a suitable set of attacks?}
    \label{fig:process}
\end{figure}

While other ranking-based semantics are described in the literature, we consider three popular gradual semantics (the weighted h-categoriser, the weighted card-based, and the weighted max-based semantics \cite{AMGOUD2022103607}). These three semantics were chosen due to their similarity to one another, and due to the way in which a similar approach can be used to solve the problem under consideration when these semantics are used. We show that when given all final acceptability degrees and initial weights, the problems for weighted h-categoriser and weighted card-based semantics are both NP-complete, while the problem can be solved in polynomial time for the weighted max-based semantics. 

The remainder of this paper is structured as follows. In Section \ref{sec:background}, we provide the necessary background to understand the remainder of the paper, introducing the h-categoriser semantics. Section \ref{sec:infer_attacks} formalises the decision problem of whether attacks can be found for gradual semantics and then demonstrates that it is NP-complete for some semantics. In Section \ref{sec:solvers} we discuss solvers for our problem, and consider related and future work (including a partial version of the problem) in Section \ref{sec:Discussion}.


\section{Background}
\label{sec:background}

As discussed above, we situate our work in the context of \emph{weighted} abstract argumentation frameworks (WAFs), which can be encoded as graphs with weighted nodes.
Each argument has an initial weight (also called ``basic score'' in \cite{rago_discontinuity-free_2016}) from $[0, 1]$.
The smaller the initial weight of an argument, the weaker the argument. 
The initial weight of an argument may represent different issues like the likelihood degree of its premises \cite{BENFERHAT1993411}, the degree of trust in its source \cite{Pereira2011164}, or an aggregation of votes provided by users \cite{leite_social_2011} among others.
In this paper, the origin of the weights and arguments is left unspecified. Similarly, arguments and attacks are considered abstract notions.

\begin{definition}
A \emph{weighted argumentation framework}  is a triple $\AF=\langle \A,\D, w \rangle$ where $\A$ is a \emph{finite} set of arguments, $\D \subseteq \A \times \A$ is a binary attack relation, and $w:\A \to [0,1]$ is a total weighting function which associates an \emph{initial weight} between $0$ and $1$ to each argument. 

\end{definition}

Given an argument $a \in \A$, we refer to $w(a)$ as $a$'s \emph{initial weight}. For any argument $a \in \A$, the set $\{b \in \A | (b,a) \in \D\}$, denoted by $\Att(a)$, contains all attackers of $a$. Similarly, we define $\Att^*(a) = \{b \in \A | (b,a) \in \D$ and $w(b) > 0\}$, i.e., the set of attackers of $a$ with strictly positive weights. When the current argumentation framework $\AF$ is not clear from the context, we will use the notation $\Att_{\AF}$ and $\Att^*_{\AF}$ respectively.

A gradual semantics $\sigma$ takes as input a weighted argumentation framework and outputs a function that associates a \emph{final acceptability degree} for each argument\footnote{Often, the final step in using a gradual semantics involves using this final acceptability degree to compute a preference ordering over arguments.}. Several such semantics have been described in the literature. In this paper, we focus on three widely used gradual semantics introduced by Amgoud et al., i.e. the weighted h-categoriser, the weighted max-based, and the weighted card-based semantics, denoted $\sigma_{HC}$, $\sigma_{MB}$ and $\sigma_{CB}$ respectively \cite{AMGOUD2022103607}.

The weighted h-categoriser considers the initial weight of the argument as well as the sum of the acceptability degrees of all its attackers to determine the acceptability degree of the argument. In this semantics, all (non-zero) attackers will be considered when determining the acceptability degree of an argument and the number of attackers is not used.

\begin{definition}
The weighted h-categoriser semantics, denoted $\sigma_{HC}$, is the function that takes as input any weighted argumentation framework $\AF=\langle \A,\D, w \rangle$ and returns the function $\sigma^\AF_{HC}:A \to [0,1]$ such that for all $a \in \A, \sigma^\AF_{HC}(a) = HC_\infty(a)$, where:

\[
HC_i(a) = \begin{cases}
  w(a) & \text{if } i=0 \\
  \frac{w(a)}{1+\sum_{b \in \Att(a)} HC_{i-1}(b)} & \text{otherwise}
\end{cases}
\]
\end{definition}

The next semantics is the weighted max-based which considers the initial weight of the argument as well as the highest acceptability degrees of its attackers to determine the acceptability degree of the considered argument. In this semantics, only the stronger attacker is considered and the number of attackers is not used.

\begin{definition}
The weighted max-based semantics, denoted $\sigma_{MB}$, is the function that takes as input any weighted argumentation framework $\AF=\langle \A,\D, w \rangle$ and returns the function $\sigma^\AF_{MB}:A \to [0,1]$ such that for all $a \in \A, \sigma^\AF_{MB}(a) = MB_\infty(a)$, where:

\[
MB_i(a) = \begin{cases}
  w(a) & \text{if } i=0 \\
  \frac{w(a)}{1+max_{b \in \Att(a)} MB_{i-1}(b)} & \text{otherwise}
\end{cases}
\]
\end{definition}

The last semantics studied in this paper is the weighted card-based which considers the initial weight of the argument, the number of its attackers as well as the sum of the acceptability degrees of its attackers to determine the acceptability degree of the considered argument. In this semantics, the number of attackers is the most important factor to determine the acceptability degree of an argument.

\begin{definition}
The weighted card-based semantics, denoted $\sigma_{CB}$, is the function that takes as input any weighted argumentation framework $\AF=\langle \A,\D, w \rangle$ and returns the function $\sigma^\AF_{CB}:A \to [0,1]$ such that for all $a\in \A, \sigma^\AF_{CB}(a)  = CB_\infty(a)$, where:

\[
CB_i(a) = \begin{cases}
  w(a) & \text{if } i=0 \\
  \frac{w(a)}{1+|\Att^*(a)|+\frac{\sum_{b \in \Att^*(a)} CB_{i-1}(b)}{|\Att^*(a)|}} & \text{otherwise}
\end{cases}
\]

Note that by convention, if $|\Att^*(a)|=0$, then we set $CB_i(a) = \frac{\sum_{b \in \Att^*(a)} CB_{i-1}(b)}{|\Att^*(a)|}=0$.

\end{definition}

We note in passing that the properties for these semantics have been researched in depth by Amgoud et al.\ in \cite{AMGOUD2022103607}. Perhaps the most important property, in the context of this paper, is the convergence of the $\sigma_{HC}, \sigma_{MB}$ and $\sigma_{CB}$ semantics for finite frameworks (see Theorem 7, 12, 17 in \cite{AMGOUD2022103607}). A proof of convergence for a broader class of semantics (including these three) was proved in the work of Oren et al. \cite{oren_analytical}. This means that for any given argumentation framework, semantics and initial weights, the final acceptability degrees of all arguments can always be computed.

\begin{example}\label{ex:ex1}
An agent may believe the following arguments.

\begin{itemize}
    \item[$a_0$:] Tomatoes older than a week can go rotten; these tomatoes are a week and a half old.
    \item[$a_1$:] Tomatoes kept in the fridge (like these ones) can last longer than a week, and so the tomatoes are not rotten.
    \item[$a_2$:] My friend ate one of the tomatoes this morning, and it tasted fine, therefore they are not rotten.
    \item[$a_3$:] My friend is not very good at discriminating whether something is rotten by taste, so the tomatoes might be rotten.
\end{itemize}

Here, $a_0$ is attacked by $a_1$ and $a_2$, while the latter is attacked by $a_3$. The agent ascribes each argument with an initial weight encoding the agent's belief in the applicability or strength of the argument, i.e., $w(a_0)=0.9$, $w(a_1) = 0.7$, $w(a_2)=0.7$, and $w(a_3)=0.6$. Fig. \ref{fig:af} illustrates this weighted argumentation framework $\AF = \langle \{a_0,a_1, a_2, a_3\}, \{ (a_1, a_0), (a_2, a_0), (a_3, a_2)\}, w \rangle$.

\begin{figure}
    \centering
    \includegraphics[width=0.17\textwidth]{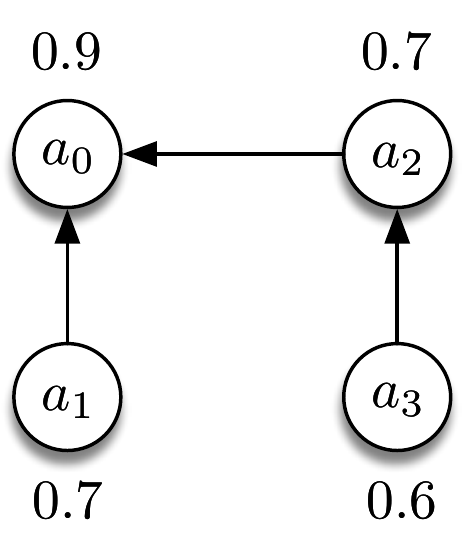}    \caption{The weighted argumentation framework from Example \ref{ex:ex1}.}
    \label{fig:af}
\end{figure}

We consider three different semantics in this paper (defined above); yielding the final acceptability degrees in Table \ref{tab:accpt_degrees}. In turn, a reasoner using the weighted h-categoriser semantics $\sigma_{HC}$ or $\sigma_{CB}$ semantics would have preferences $a_1 \succ a_3 \succ a_2 \succ a_0$, while if it were to use $\sigma_{MB}$  preferences would be $a_1 \succ a_3 \succ a_0 \succ a_2$. Note that while the $\sigma_{HC}$ and $\sigma_{CB}$ semantics yield similar preference orderings, the agent would be less sure of its conclusions in the latter case (due to the smaller difference in weights).

\begin{table}[]
    \centering
    \begin{tabular}{|c|l|l|l|}
\hline
Argument & $\sigma_{HC}^\AF(a_i)$ & $\sigma_{MB}^\AF(a_i)$ & $\sigma_{CB}^\AF(a_i)$ \\
\hline
$a_0$ & 0.421 & 0.529  & 0.258  \\
$a_1$ & 0.7  & 0.7 & 0.7\\
$a_2$ & 0.438 & 0.438 & 0.269\\
$a_3$ & 0.6 & 0.6  & 0.6 \\
\hline
\end{tabular}
\caption{\label{tab:accpt_degrees} Final acceptability degrees for each argument of Example \ref{ex:ex1} (Fig. \ref{fig:af}) for $\sigma_{HC}, \sigma_{MB},$ and $\sigma_{CB}$.}
\end{table}

\end{example}

\section{Inferring Attacks}
\label{sec:infer_attacks}

In this paper, we focus on the problem of identifying whether we can infer suitable attacks between arguments given their initial weights and desirable final acceptability degrees (as shown at the bottom of Fig. \ref{fig:process}). 
However, we ground most of our discussion in the closely related decision problem. Given a gradual argumentation semantics, a set of arguments, and associated information about these arguments (i.e., initial weights and some or all desirable final acceptability degrees), can we determine the attacks between arguments that --- according to the semantics --- lead to the desirable final acceptability degrees from the initial weights.  

In this section, we consider the \emph{complete problem}, where the initial weights and desirable final acceptability degrees for \emph{all} arguments are known. We formalise this problem as follows.

\begin{problem}
The decision problem $\Dec^x_c$ is: \emph{Given a set of arguments $\A$, a gradual semantics $\sigma_x$, a weighting function $w$, and a total desired final acceptability degree function $S:\A \to [0,1]$, is there a set of attacks $\D \subseteq \A \times \A$ such that for all $a \in \A, \sigma^\AF_x(a)=S(a)$, where $\AF=\langle \A, \D, w \rangle$?}
\end{problem}

We investigate this decision problem for the three gradual semantics $\sigma_{x}$, for $x \in \{HC, CB, MB\}.$

\begin{proposition}
$\Dec_c^{MB}$ is polynomial and can be decided in $O(n)$ time and space.
\label{prop:MB-poly}
\end{proposition}

\begin{proof}
We create the set $L = \{1+S(b) \mid b \in \A\}$, for which an $O(1)$ lookup can be performed. We can then use the following trivial algorithm to solve $\Dec_c^{MB}$:

\begin{algorithmic}[1]
\ForAll{$a \in \A$}
\If{($w(a) = 0$ and $S(a) \neq 0$) or ($w(a) \neq 0$ and $S(a) = 0$)} \label{line-cond-zero}
\State \Return False \Comment{A zero initial weight cannot lead to a non-zero final acceptability degree (and vice-versa).} \label{contradiction-zero}
\EndIf

\If{$S(a) \neq 0$ and $w(a)/S(a) \notin L$} \label{cond-exist-max-attacker}
  \State \Return False \Comment{Contradiction by definition of the $\sigma_{MB}$ semantics.} \label{return-false-no-attack}
\EndIf
\EndFor
\State \Return True
\end{algorithmic}
\end{proof}

In the algorithm above, line \ref{line-cond-zero} checks whether arguments with zero (resp. non-zero) initial weights and non-zero (resp. zero) desired final acceptability degrees exist; the presence of such arguments would mean that the weighted-max-based semantics cannot be satisfied (line \ref{contradiction-zero}). Line \ref{cond-exist-max-attacker} then checks the existence of an attacking argument with a suitable acceptability degree (via a lookup in the set $L$) to satisfy the weighted-max-based semantics in $O(1)$ time. Since lines 1-8 iterate over all arguments, the time complexity is $O(n)$; similarly, storing $L$ takes $O(n)$ space.

\begin{figure}[h]
\centering
\begin{tikzpicture}[scale=0.1]
\tikzstyle{every node}+=[inner sep=0pt]
\draw [black] (16.9,-30.8) circle (3);
\draw (16.9,-30.8) node[label={[label distance=10] 90: 1}, label={[label distance=10] 270: \color{blue} $\frac{nm^*}{0.4T+nm^*}$}] {$a_0$};
\draw [black] (36.7,-15.5) circle (3);
\draw (36.7,-15.5) node[label={[label distance=10] 180:  \color{blue}$\frac{0.4m_1}{nm^*}$}] {$a_1$};
\draw [black] (36.7,-25.5) circle (3);
\draw (36.7,-25.5) node[label={[label distance=10] 180:  \color{blue}$\frac{0.4m_2}{nm^*}$}] {$a_2$};
\draw [black] (36.7,-46.7) circle (3);
\draw (36.7,-46.7) node[label={[label distance=10] 180: \color{blue} $\frac{0.4m_n}{nm^*}$}] {$a_n$};
\draw (36.7,-34) node {$\vdots$};
\draw [black] (58.5,-17.5) rectangle (52.9,-12.5);
\draw (55.9,-15.5) node {$m_1$};
\draw [black] (58.5,-27.5) rectangle (52.9, -22.5);
\draw (55.9,-25.5) node {$m_2$};

\draw [black] (58.5,-48.7) rectangle (52.9,-43.7);
\draw (55.9,-46.7) node {$m_n$};
\draw [black] (7.8,-32.8) rectangle (2.8, -27.8);
\draw (5.8,-30.8) node {$T$};

\draw [black,dashed] (7.8,-30.8) -- (13.9,-30.8);

\draw [black,dashed] (52.9,-15.5) -- (39.7,-15.5);

\draw [black,dashed] (52.9,-25.5) -- (39.7,-25.5);

\draw [black,dashed] (52.9,-46.7) -- (39.7,-46.7);

\end{tikzpicture}
\caption{Graphical representation of the polynomial transformation (dashed lines) of an $\Ssp$ instance (rectangular nodes) into an instance of $\Dec_c^{HC}$ (circular nodes).  The initial weight of $a_0$ is 1 (shown in black above $a_0$), and of $a_1$ to $a_n$, is equal to their final acceptability degree (where final acceptability degrees are shown in blue). The initial weights of the latter arguments have thus been omitted.}
\label{fig:ssp_to_HC}
\end{figure}
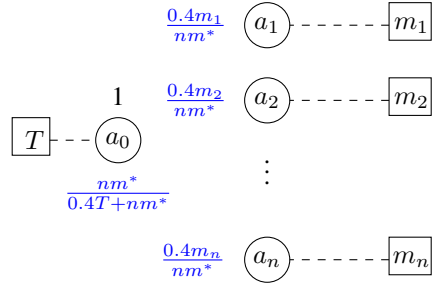

\begin{proposition}
$\Dec_c^{HC}$ is NP-complete.
\label{proposition_DEC_HC}
\end{proposition}

\begin{proof}
\label{proof:DEC_HC}
    We show that $\Dec_c^{HC}$ is NP-complete using the subset-sum problem ($\Ssp$). Formally, $\Ssp$ is defined as answering the following question. Consider a multiset of positive numbers $M$ and a number $T \in \mathbb{R}^+$, is there a subset $M' \subseteq M$ such that $\sum\limits_{m' \in M'} m' = T$?
    
    To prove $\Dec_c^{HC}$ is NP-complete we must demonstrate that it is in NP and identify a polynomial time reduction from $\Ssp$ to $\Dec_c^{HC}$.
    
    To demonstrate that $\Dec_c^{HC}$ belongs to NP we observe that the certificate of the problem is a set of attacks. Given a set of attacks, we can check in polynomial time whether $S(a) = w(a)/ (1+ \sum_{b \in \Att(a)} S(b))$ for all $a \in \A$ (as we have all initial weights and the desired final acceptability degrees), therefore the problem is in NP.

    Turning to the reduction, we begin by reducing $\Ssp$ to $\Dec_c^{HC}$. Let us assume we have a multi-set of positive numbers $M = \{ m_1, \dots, m_n\}$ and $T \in \mathbb{R}^+$. We denote $\sum_i m_i$ by $m^*$. We then create a set of $n+1$ arguments $\A = \{a_0, a_1, a_2, \dots, a_n\}$ such that for all $i \in \{1, \dots , n \}$, $S(a_i) = w(a_i) = (0.4m_i)/ (nm^*)$, and we set $w(a_0) = 1$ and $S(a_0) = nm^* / (0.4T + nm^*)$. 
    This transformation is represented in Fig. \ref{fig:ssp_to_HC}.

\begin{example} \label{example:ssp2dechc}
Consider the $\Ssp$ instance where $T=100$ and $M=\{23,94,1,37,40\}$.  Here, $n=5$ and $m^*=195$. Using our transformation, we obtain the  arguments, initial weights, and final acceptability degrees shown in Table \ref{tab:examplessp}.

\end{example}

\begin{table}[h]
\centering
\begin{tabular}{|c|c |c |}
\hline
\textbf{Argument} & \textbf{Initial} & \textbf{Desired final } \\

 & \textbf{Weight} & \textbf{acceptability degree}\\
\hline
$a_0$     &  1 & 0.96059\\
\hline
$a_1$     &  \multicolumn{2}{c|}{0.00944}\\
\hline
$a_2$     &  \multicolumn{2}{c|}{0.03856}\\
\hline
$a_3$     &  \multicolumn{2}{c|}{0.00041}\\
\hline
$a_4$     &  \multicolumn{2}{c|}{0.01518}\\
\hline
$a_5$     &  \multicolumn{2}{c|}{0.01641}\\
\hline
\end{tabular}
\caption{\label{tab:examplessp} The arguments created from Example \ref{example:ssp2dechc} using the reduction described in Proposition \ref{proposition_DEC_HC}.}
\end{table}

    \begin{itemize}
        \item  We now demonstrate that --- using the above reduction --- a solution to $\Ssp$ exists only if a solution to $\Dec_c^{HC}$ exists.
        If there exists an $M'  \subseteq M$ such that $\sum\limits_{m' \in M'} m' = T$, then $\D = \{ (f(m), a_0 ) \mid m \in M' \}$ is a set of attacks such that for all $a \in \A, \sigma^\AF(a) = S(a)$, where $f$ associates the corresponding argument and $\AF = \langle \A, \D, w \rangle$. Indeed, we have for all $i \in \{1, \dots, n\}, \sigma^\AF_{HC}(a_i) = S(a_i)$ since $a_i$ is not attacked and:
    
    \begin{align*}
    \sigma^\AF_{HC}(a_0) &= \frac{1}{1 + \sum_{m \in M'} S(f(m))} \\
    &= \frac{1}{1 + \frac{0.4T}{nm^*}} \\
    &= \frac{nm^*}{nm^* + 0.4T} \\
    &= S(a_0)
    \end{align*}

    \item We now show the reduction in the other direction --- a solution to $\Dec_c^{HC}$ exists only if a solution to $\Ssp$ can be found via the above reduction. If there exists $\D \subseteq \A \times \A$ such that  for all $b \in \A, \sigma^\AF(b) = S(b)$, then $M' = \{ (S(b)nm^*)/0.4 \mid (b,a_0) \in \D \} \subseteq M$ is such that $\sum\limits_{m' \in M'} m' = T$, where $\AF = \langle \A, \D, w \rangle$. Indeed:

    \begin{align*}
        \sum\limits_{m' \in M'} m' &= \left(\sum\limits_{(b,a_0) \in \D} S(b)\right) \frac{nm^*}{0.4}\\
        &= \left(\sum\limits_{(b,a_0) \in \D} \sigma^\AF_{HC}(b)\right) \frac{nm^*}{0.4}\\
        &= \left(\frac{1}{\frac{nm^*}{0.4T+nm^*}} -1\right) \frac{nm^*}{0.4}\\
        &= \left(\frac{0.4T}{nm^*}\right)  \frac{nm^*}{0.4}\\
        &= T\\
    \end{align*}

    Note that $(a_0, a_0) \notin \D$. Indeed, if $(a_0, a_0) \in \D$ then we have that $\sigma_{HC}^\AF(a_0) = 1/(1 + \sigma_{HC}^\AF(a_0) + Y)$, where $Y = \sum_{b \in Att(a_0) \setminus \{a_0\}} \sigma_{HC}^\AF(b)$. Hence, ($\sigma_{HC}^\AF(a_0)^2 + \sigma_{HC}^\AF(a_0) -1)/(-\sigma_{HC}^\AF(a_0)) = Y$ and since $Y \geq 0$, we conclude that $0 \leq \sigma_{HC}^\AF(a_0) \leq \frac{-1 + \sqrt{5}}{2} \simeq 0.618$. However, we get a contradiction as $S(a_0) = \sigma_{HC}^\AF(a_0)$ and:

    \begin{align*}
        \frac{T}{nm^*} &\leq 1\\
        \frac{0.4T}{nm^*} &\leq 0.4\\
        \frac{0.4T+nm^*}{nm^*} &\leq 1.4\\
        S(a_0) &\geq 0.714\\
    \end{align*}
 \end{itemize}

We have proved that $\Dec_c^{HC}$ is in NP and that $\Ssp$ is polynomial time reducible to $\Dec_c^{HC}$ (the size of the argumentation framework produced is polynomial with respect to the size of the $\Ssp$ instance), therefore $\Dec_c^{HC}$ is NP-complete.
\end{proof}

\begin{example}[Cont'd Example \ref{example:ssp2dechc}]
We see that the subset sum problem has a solution (using values 23,37 and 40). Analogously, a solution to $\Dec_c^{HC}$ exists by having arguments $a_1, a_4$ and $a_5$ attack argument $a_0$, as shown in Fig. \ref{fig:ssp_to_HC_example}.

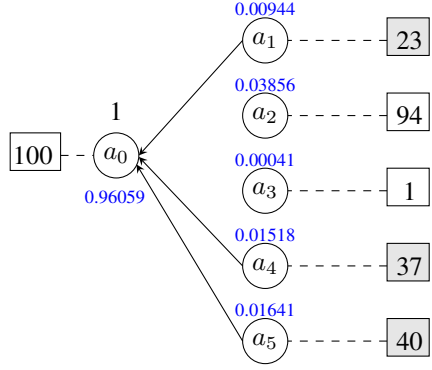
\begin{figure}
\centering
\begin{tikzpicture}[scale=0.1]
\tikzstyle{every node}+=[inner sep=0pt]
\draw [black] (16.9,-30.8) circle (3);
\draw (16.9,-30.8) node[label={[label distance=10] 90: 1}, label={[label distance=10] 268: \color{blue} \scriptsize 0.96059}] {$a_0$};
\draw [black] (36.7,-15.5) circle (3);
\draw (36.7,-15.5) node[label={[label distance=6] 90:  \color{blue} \scriptsize 0.00944}] {$a_1$};

\draw [black] (36.7,-25.5) circle (3);
\draw (36.7,-25.5) node[label={[label distance=6] 90:  \color{blue} \scriptsize 0.03856}] {$a_2$};

\draw [black] (36.7,-35.5) circle (3);
\draw (36.7,-35.5) node[label={[label distance=6] 90:  \color{blue} \scriptsize 0.00041}] {$a_3$};

\draw [black] (36.7,-45.5) circle (3);
\draw (36.7,-45.5) node[label={[label distance=6] 90:  \color{blue} \scriptsize 0.01518}] {$a_4$};

\draw [black] (36.7,-55.5) circle (3);
\draw (36.7,-55.5) node[label={[label distance=6] 90:  \color{blue} \scriptsize 0.01641}] {$a_5$};

\draw [black, stealth-] (19.9,-30.8) -- (33.7,-15.5);
\draw [black, stealth-] (19.9,-31) -- (33.7,-45.5);
\draw [black, stealth-] (19.6,-32) -- (33.7,-55.5);

\draw [black, fill=gray!20] (58.5,-17.5) rectangle (52.9,-12.5);
\draw (55.9,-15.5) node {23};
\draw [black] (58.5,-27.5) rectangle (52.9, -22.5);
\draw (55.9,-25.5) node {94};

\draw [black] (58.5,-37.5) rectangle (52.9, -32.5);
\draw (55.9,-35.5) node {1};

\draw [black, fill=gray!20] (58.5,-47.5) rectangle (52.9, -42.5);
\draw (55.9,-45.5) node {37};

\draw [black, fill=gray!20] (58.5,-57.5) rectangle (52.9, -52.5);
\draw (55.9,-55.5) node {40};

\draw [black] (9.5,-32.8) rectangle (2.8, -27.8);
\draw (5.8,-30.8) node {100};

\draw [black,dashed] (9.5,-30.8) -- (13.9,-30.8);
\draw [black,dashed] (52.9,-15.5) -- (39.7,-15.5);
\draw [black,dashed] (52.9,-25.5) -- (39.7,-25.5);
\draw [black,dashed] (52.9,-35.5) -- (39.7,-35.5);
\draw [black,dashed] (52.9,-45.5) -- (39.7,-45.5);
\draw [black,dashed] (52.9,-55.5) -- (39.7,-55.5);

\end{tikzpicture}
\caption{Representation of how a solution to an $\Ssp$ instance (represented with the gray rectangles) can be obtained from a solution to a corresponding instance from $\Dec^{HC}_c$ (the attacks drawn). Blue values for $a_1 \ldots a_5$ indicate both inital weights and final acceptability degrees, and denote final acceptability degrees for $a_0$, which has an initial weight of 1.}
\label{fig:ssp_to_HC_example}
\end{figure}

\end{example}

\begin{proposition}
$\Dec_c^{CB}$ is NP-complete.
\end{proposition}

 \begin{proof}

Similar to the previous proof, we show that $\Dec_c^{CB}$ is NP-complete using the $k$-subset-sum problem ($\kSsp$). Formally, $\kSsp$ is defined as answering the following question. Consider a multiset of positive numbers $M$ and a number $T \in \mathbb{R}^+$, is there a subset $M' \subseteq M$ such that $\sum_{m' \in M'} m' = T$ and $|M'| = k$?




We first observe that the certificate of the problem $\Dec_c^{CB}$ is a set of attacks. Given a set of attacks, we can check in polynomial time whether $S(a) = w(a) / (1+ |\Att^*(a)| + \sum_{b \in \Att^*(a)} S(b)/|\Att^*(a)|)$ for all $a \in \A$, meaning that the problem is in NP. 

We now reduce $\kSsp$ to $\Dec_c^{CB}$. Let us assume we have a multi-set of positive numbers $M = \{ m_1, \dots, m_n\}$, $ 1 \leq k \leq n$, and $T \in \mathbb{R}^+$. We denote by $m^* = \max_{m_i \in M} m_i$ and $u = \frac{\sqrt{k^2+2k+4/k+1}-k -1}{3}$ (note that $u\in ]0, (2\sqrt{2}-2)/3]$ when $k\geq 1$). We create $n+1$ arguments $\A = \{ a_0, a_1, a_2, \dots, a_n\}$ such that for all $i \in \{ 1, \dots, n \}, S(a_i) = w(a_i) = um_i /m^*$ and we set $w(a_0) = 1$ and $S(a_0) = 1 / (1+k + Tu/(km^*)) $.

\begin{itemize}
    \item If there exists an $M' \subseteq M$ such that $\sum_{m' \in M'} m' = T$ and $|M'| = k$, then $\D = \{ (f(m), a_0) | m \in M' \}$ is a set of $k$ attacks such that for all $a \in \A, \sigma^\AF(a) = S(a),$ where $\AF = \langle \A, \D, w \rangle$. Indeed, we have for all $i \in \{1, \dots, n \}, \sigma_{CB}^\AF(a_i) = S(a_i)$ since $a_i$ is not attacked and:

\begin{align*}
    \sigma_{CB}^\AF(a_0) &= \frac{1}{1 + k + \frac{\sum_{m \in M'}S(f(m))}{k}} \\
    &= \frac{1}{1 + k + \frac{\frac{Tu}{m^*}}{k}}\\
    &= \frac{1}{1 + k + \frac{Tu}{km^*}}
\end{align*}

\item Assume we have a solution to $\Dec_c^{CB}$, i.e. we have $\D \subseteq \A \times \A$ such that for all $b \in \A, \sigma^{\AF}(b) = S(b)$. We know that $a_0$ is attacked by $k$ arguments in $\D$ as its value lies between $[1/(k+2), 1/(k+1)]$ and that arguments $a_1$ to $a_n$ are not attacked (except if they or their attackers have initial weights of $0$).
We can build $M' = \{ S(b)*m^*/u | (b,a_0) \in \D \} \subseteq M$ such that $|M'| = k$ and $\sum_{m' \in M'} m' = T$. Indeed:

\begin{align*}
    \sum_{m' \in M'} m' &= \left( \sum_{(b,a_0) \in \D} S(b) \right) \frac{m^*}{u} \\
    &= \left( \frac{Tu}{m^*} \right)\frac{m^*}{u} \\
    &= T
\end{align*}

We show that $(a_0, a_0) \notin \D$ by contradiction. Assume we have this self-attack, then the maximum value with self attack possible for $\sigma_{CB}^\AF(a_0)$ is determined by computing the unique fixed-point of the function $f(x) = 1/ (1+k + x/k)$ which is:

$$0 < \frac{-k^2 + \sqrt{k(k^3+2k^2+k+4)}-k}{2} \leq 1$$

Moreover, it holds that:

\begin{align*}
    \frac{-k^2 + \sqrt{k(k^3+2k^2+k+4)}-k}{2} &< \frac{1}{1+k+u}\\
    &< \frac{1}{1+k+\frac{Tu}{km^*}} \\
    &< S(a_0)
\end{align*}

This is a contradiction with $S(a_0) = \sigma_{CB}^\AF(a_0)$.




\end{itemize}

\end{proof}

To recap, our results show that $\Dec^{MB}_c$ is polynomial (and indeed, can be solved in linear time), while $\Dec^{HB}_c$ and $\Dec^{CB}_c$ are both NP-complete. This latter result was obtained by a reduction to a variant of the subset-sum problem.

\section{Identifying Solutions}
\label{sec:solvers}

Rather than simply considering the decision problem, it is useful --- if possible --- to be able to infer the attacks induced by some set of initial weights and final acceptability degrees. We consider each semantics individually.

\subsection{The Weighted Max-based Semantics}

For the weighted max-based semantics, it is trivial to modify the algorithm from Proposition \ref{prop:MB-poly} to infer a suitable set of attacks (when possible). Rather than simply returning True or False, we return, by modifying lines \ref{cond-exist-max-attacker} and \ref{return-false-no-attack}, any set of attacks $X$ that satisfy the condition that for all $a \in \A$, there exists $(x, a) \in X$ such that $w(a)/S(a)=(1+S(x))$. Thus, we note that this semantics does little to constrain the attacks in the framework. If we discover that there is an attack from argument $a_i$ to argument $a_0$ in a solution, then under the max-based semantics, we can expand this solution by adding additional attacks from any argument $a_j$ to $a_0$, where $S(a_j) \leq S(a_i)$, without affecting the result.

\begin{proposition}[Solution Expansion]
\label{prop:sol-exp-MB}
Given a set of arguments $\A$, a weighting function $w : \A\to [0,1]$, and a desirable final acceptability degree function $S: \A \to [0,1]$. If $\D \subseteq \A \times \A$ is such that for all $a \in \A$, $\sigma_{MB}^\AF(a) = S(a)$, where $\AF = \langle \A, \D,w \rangle$ then for all $a \in \A,\sigma_{MB}^{\AF'}(a) = S(a)$, where $\AF' = \langle \A, \D',w \rangle, \D' = \D \cup X$ and $X \subseteq \{ (a_i, a_j) \mid \exists (a_k, a_j) \in \D$ and $S(a_i) \leq S(a_k)\} \cup \{ (a_i, a_j) \mid w(a_j) = 0 \}$.
We say that $\D'$ is an expansion of $\D$.
\end{proposition}

Similarly, if we have a set of attacks that achieve the desirable final acceptability degrees for all arguments and there is an attack from $a_i$ to $a_0$, then under the max-based semantics, we can contract this solution by removing attacks from any other argument $a_j$ to $a_0$, where $S(a_j) \leq S(a_i)$.


\begin{proposition}[Solution Contraction]
\label{prop:sol-contraction}
Given a set of arguments $\A$, a weighting function $w : \A\to [0,1]$, and a desirable final acceptability degree function $S: \A \to [0,1]$. If $\D \subseteq \A \times \A$ is such that for all $a \in \A$, $\sigma_{MB}^\AF(a) = S(a)$, where $\AF = \langle \A, \D,w \rangle$ then for all $a \in \A,\sigma_{MB}^{\AF'}(a) = S(a)$, where $(a_i, a_0) \in \D, \AF' = \langle \A, \D',w \rangle, \D' = \D \setminus X$ and $X \subseteq \{ (a_j, a_0) \in \D \mid a_j \neq a_i$ and $S(a_j) \leq S(a_i)\} \cup \{ (a_j, a_0) | w(a_0) = 0 \}$.
We say that $\D'$ is a contraction of $\D$ (on $(a_i, a_0) \in \D$).
\end{proposition}

\begin{figure}
    \centering
    \includegraphics[width=0.33\textwidth]{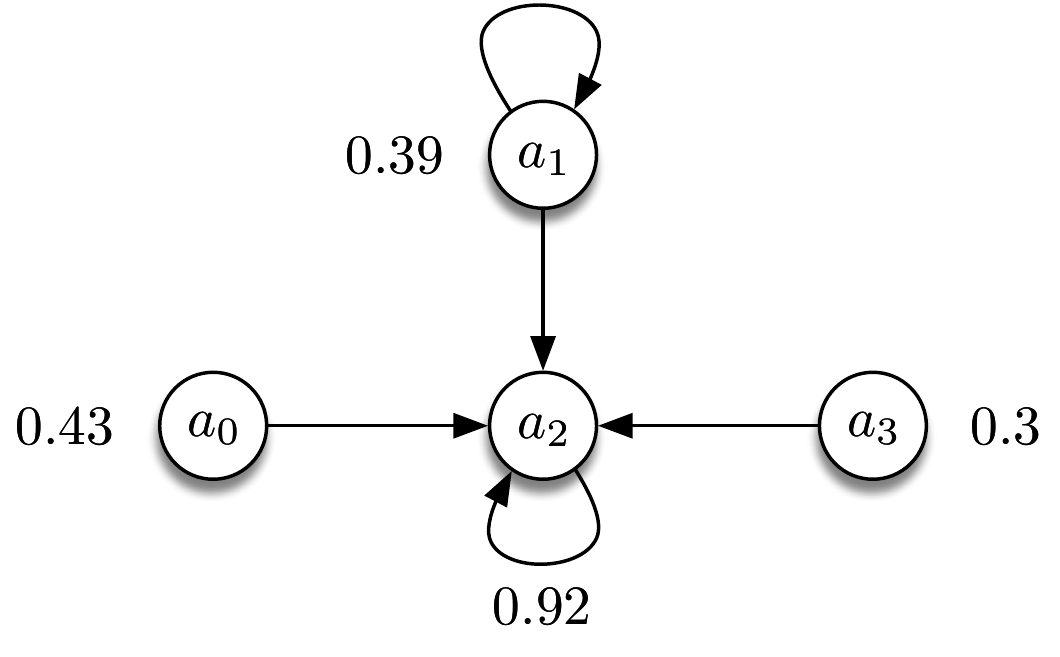}
    \caption{The argumentation framework of Example \ref{ex:eg4}.}
    \label{fig:af2}
\end{figure}

\begin{example}\label{ex:eg4}
Let us consider $\A$, $w$, $\D$ as shown in Figure \ref{fig:af2} and the desirable final acceptability degree function $S$ such that $S(a_0) = 0.43$, $S(a_1) = 0.30$, $S(a_2) = 0.58$, and $S(a_3) = 0.30$. A contraction $\D'$ of $\D$ on $(a_2, a_2)$ will remove the attacks $(a_0, a_2), (a_1, a_2)$, and $(a_3, a_2)$ from $\D$ without changing the acceptability degrees under the max-based semantics. Likewise, an expansion $\D''$ of $\D'$ can be obtained by adding the attack $(a_0, a_2)$ as this does not change the acceptability degrees. This is represented in Fig. \ref{fig:contraction}.

\begin{figure}
\centering
\includegraphics[width=0.82\textwidth]{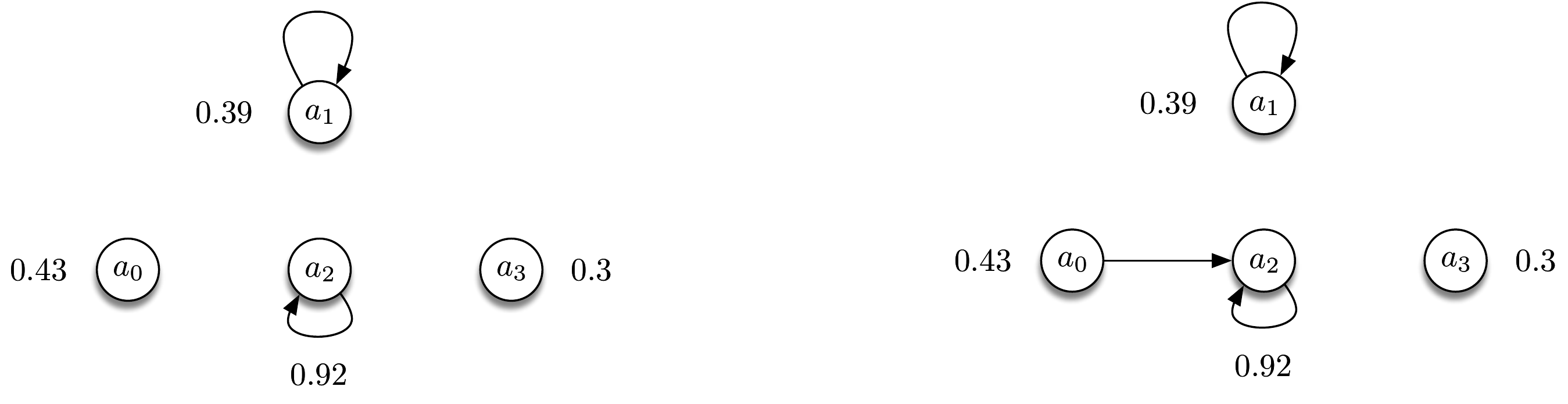}

\caption{\label{fig:contraction} Representation of $\AF' = \langle \A, \D', w \rangle$ (left), where $\D'$ is a contraction of $\D$ on $(a_2, a_2)$ and $\AF'' = \langle \A, \D'', w \rangle$ (right), where $\D''$ is an expansion of $\D'$.}
\end{figure}

\end{example}

Once a solution for the max-based semantics is found, we can reach all the other solutions by expanding this initial solution once and then successively contracting it.

\begin{proposition}
Given a set of arguments $\A$, a weighting function $w : \A\to [0,1]$, and a desirable final acceptability degree function $S: \A \to [0,1]$. If $\D,\D' \subseteq \A \times \A$ are such that for all $a \in \A$, $\sigma_{MB}^\AF(a) = S(a)$, $\sigma_{MB}^{\AF'}(a) = S(a)$ where $\AF = \langle \A, \D,w \rangle$ and $\AF' = \langle \A, \D',w \rangle$ then there exists $\D^* \subseteq \A \times \A$, such that $\D^*$ is an expansion of $\D$, and a sequence of sets of attacks $(\D_1, \D_2, \dots, \D_n)$, where $\D_1 = \D^*$, $\D_n = \D'$ and for all $1 \leq i \leq n-1, \D_{i+1}$ is a contraction of $\D_{i}$.
\label{prop:solution-MB}
\end{proposition}

\begin{proof}   
Let us consider $\A = \{a_1, a_2, \dots, a_n\}$ and two arbitrary solutions $\D, \D' \subseteq \A \times \A$ for $\Dec_c^{MB}$. We prove this proposition by construction.
By definition, for all $a \in \A, \sigma_{MB}^\AF(a) = \sigma_{MB}^{\AF'}(a) = S(a)$, where  $\AF = \langle \A, \D,w \rangle$ and $\AF' = \langle \A, \D',w \rangle$.
This means that for all $a \in \A$, such that $w(a) \neq 0, \max_{ b \in \Att_\AF(a)} S(b) = \max_{ b \in \Att_{\AF'}(a)} S(b)$.
Thus, $\D^* = \D \cup \D'$ is an expansion of $\D$.
Then, for all $1 \leq i \leq n$, it holds that $\D_{i+1} = \D_i \setminus \{ (a_j, a_i)\in \D | (a_j ,a_i) \notin \D' \}$ is a contraction of $\D_{i}$, where $\D_1 = \D^*$. It is clear that $\D_{n+1}= \D'$ as all attacks in $\D \setminus \D'$ have been removed from $\D^*$.
\end{proof}

\subsection{The Weighted H-Categoriser Semantics}

Let us consider the case of the weighted h-categoriser semantics. 
Given an input $\A = \{a_1, \dots, a_n\}$, an initial weight function $w$, and a desirable final acceptability degree function $S$; we seek a set of attacks $\D \subseteq \A \times \A$ such that for all $a \in \A, \sigma_{HC}^\AF(a) = S(a)$, where $\AF = \langle\A, \D, w \rangle$.

Recall from the definition of the weighted h-categoriser semantics that any solution must mean that the following equation holds.

$$S(a_i) = \frac{w(a_i)}{1+\sum_{b \in \Att(a_i)} S(b)}$$

Moreover, if $w(a_i) \neq 0$, we can rearrange the previous equation to obtain:

$$\sum_{b \in \Att(a_i)} S(b) = \frac{w(a_i)-S(a_i)}{S(a_i)}$$

For the complete problem, we are given $w(a_i)$ and $S(a_i)$ for all $a_i \in \A$, allowing us to compute $\sum_{b \in \Att(a_i)} S(b) $ via the above equation. In other words, for every argument $a_i$, we can compute a target value $T=\sum_{b \in \Att(a_i)} S(b)$ for which the final acceptability degrees of all its attackers must sum up to. By considering the multi-set $M = \{ S(a_j) \mid a_j \in \A\}$, we must solve a single subset-sum problem to identify the attackers for a single argument $a_i$. Extending this to all $n=|\A|$ arguments in our framework, we can therefore identify all attacks in the framework by solving $n$ versions of the subset-sum problem. Algorithm \ref{alg:decHC} provides the pseudo-code for this approach.

\begin{algorithm}
\begin{algorithmic}[1]
\Require $\A$ a set of arguments
\Require $w:\A \to [0,1]$ the initial weights for each argument
\Require $S:\A \to [0,1]$ the desired final acceptability degrees for each argument
\Function{Solve$_{HC}$}{$\A,w,S$}
\State $D=\{\}$
\State $M \gets \{S(b)|b \in \A \}$
\ForAll{$a \in \A$}

    \If{($w(a) = 0 $ and $S(a) \neq 0$) or ($w(a) \neq 0 $ and $S(a) = 0$)}
        \State \Return False
    \EndIf

    \If{$w(a) \neq 0$ and $w(a) \neq S(a)$}
  \State $T \gets \frac{w(a)-S(a)}{S(a)}$
  \If{\Call{SSP}{$T,M$} = False} \label{alg:call-ssp}
    \State \Return False
  \EndIf    
  \State $D=D \cup \{ (b,a) | S(b) \in $\Call{SSP}{$T,M$}$ \}$ \label{alg:call-ssp2}
  \EndIf
\EndFor
\State \Return $D$ 
\EndFunction
\end{algorithmic}
\caption{Procedure to solve $\Dec_c^{HC}$. SSP (line \ref{alg:call-ssp}, \ref{alg:call-ssp2}) calls a subset-sum solver which returns the elements of $M$ that sum up to $T$, or False if no such elements can be found.}
\label{alg:decHC}
\end{algorithm}

We note in passing that the standard version of the subset-sum problem assumes that $T$ and elements of $M$ are all integers. If we assume that all initial weights and final acceptability degrees are rational, we can easily transform our computed $T$ and $M$ values to integers. The most common --- dynamic programming --- approach to solving standard subset-sum can run in pseudo-polynomial time. However, the integers we obtain become very large very quickly, making this approach impractical. Future work will consider using state-of-the-art techniques for solving subset-sum over real numbers instead, e.g. the recent FPTAS of Costandin \cite{costandin2021fptas}. In addition, having transformed our problem to an instance of the subset-sum problem opens up the possibility of transforming the problem to other NP-complete problems for which efficient solvers exist, such as satisfiability.

Moreover, once a set of attacks is found to be a solution, we can obtain other solutions by replacing the attacks to a single argument $x$ by other attacks to this argument so that the sum of the degree of the attackers remains the same. Given the simplicity of this proposition, we do not provide a proof for it here.

\begin{proposition}
\label{prop:solution-HC}
Consider two weighted argumentation frameworks $\AF = \langle \A, \D, w \rangle$, $\AF' = \langle \A, \D', w \rangle$, and a desirable final acceptability degree function $S: \A \to [0,1]$ such that $\sigma^{\AF}_{HC}(a)=S(a)$ for any $a \in \A$. If all of the following hold
\begin{itemize}
    \item $x \in \A$
    \item $Z \subseteq \A \times \{x\}$
    \item $\sum_{(z,x) \in Z} S(z) = (w(x)-S(x))/S(x)$ if $w(x) \neq 0$
    \item $\D'=(\D \cap (\A \times (a \setminus \{x\} ))) \cup Z$
\end{itemize}
then for all $a \in \A, \sigma^{\AF'}_{HC}(a)=S(a)$.
\end{proposition}


\subsection{The Weighted Cardinality-based Semantics}

We can adopt a similar approach for $\sigma_{CB}$ as was done for $\sigma_{HC}$. We begin by assuming that there are $k$ attacks from arguments with non-zero initial weight against an argument $a_i$. Any solution must mean that the following equation holds:

$$S(a_i) = \frac{w(a_i)}{1+k+\frac{1}{k}\sum_{b \in \Att(a_i)} S(b)}$$

If $w(a_i) \neq 0$, rearranging gives us the equation:
$$
\sum_{b \in \Att(a_i)} S(b) = -\frac{k(S(a_i)k+S(a_i)-w(a_i))}{S(a_i)}
$$

Again, all values on the right-hand side of the equation are given for the complete problem given, allowing us to compute the value of the left-hand side for each argument. Assuming a non-zero number of attacks, while we could (naively) iterate over all $k=1 \ldots n$ to identify the number of attacks against an argument, and thereby determine $k$, we observe that if 
$S(b) = \epsilon >0 $, $\lim{\epsilon} \to 0$,
for all $b \in \Att(a_i)$ the semantic equation reduces to the following.
$$S(a_i)=\frac{w(a_i)}{1+k}$$
Solving for $k$, we see that:
$$k=\frac{w(a_i)-S(a_i)}{S(a_i)}$$
Now assume instead that $S(b)=1$ for all $b \in \Att(a_i)$. Performing the same operations means that:
$$k=\frac{w(a_i)-2S(a_i)}{S(a_i)}$$

These two equations serve as an upper and lower bound for $k$. Since $k$ must be an integer, we can avoid the need to iterate over the number of attacks by computing $k$ as follows.
$$k = \lfloor \frac{w(a_i)-S(a_i)}{S(a_i)} \rfloor = \lceil \frac{w(a_i)-2S(a_i)}{S(a_i)} \rceil$$

The pseudo-code for this approach is described in Algorithm \ref{alg:dec_CB}.

\begin{algorithm}
    \begin{algorithmic}[1]
\Require $\A$ a set of arguments
\Require $w:\A \to \mathbb{Q}$ the initial weights for each argument
\Require $S:\A \to \mathbb{Q}$ the final acceptability degree for each argument
\Function{Solve$_{CB}$}{$\A,w,S$}
\State $D=\{\}$
\State $M \gets \{S(b)|b \in \A \}$
\ForAll{$a \in \A$}

    \If{($w(a) = 0 $ and $S(a) \neq 0$) or ($w(a) \neq 0 $ and $S(a) = 0$)}
        \State \Return False
    \EndIf

    \If{$w(a) \neq 0$ and $w(a) \neq S(a)$}
        \State $k \gets \lfloor \frac{w(a)-S(a)}{S(a)} \rfloor$
        \State $T \gets -\frac{k(S(a)k+S(a)-w(a))}{S(a)}$
        \If{\Call{kSSP}{$T,M,k$} = False} \label{alg:call-kssp}
            \State \Return False
        \EndIf
        \State $D \gets D \cup \{(b,a)|S(b) \in $\Call{kSSP}{$T,M,k$}$ \}$ \label{alg:call-kssp2}
    \EndIf
\EndFor  
\State \Return $D$
\EndFunction
\end{algorithmic}
\caption{Procedure to solve $\Dec_c^{CB}$. kSSP (line \ref{alg:call-kssp}, \ref{alg:call-kssp2}) calls a subset-sum solver which returns $k$ elements of $M$ which sum up to $T$.}
\label{alg:dec_CB}
\end{algorithm}

Since $k$ is known, the $\kSsp$ solver called in Algorithm \ref{alg:dec_CB} could potentially be more efficient than that used in the h-categoriser semantics. 

Analogous to the result from Proposition \ref{prop:solution-HC}, under the cardinality semantics, once a set of attacks is found to be a solution, we can obtain other solutions by replacing the $k$ attacks on a single argument $x$ by another set of $k$ attacks to this argument so that the sum of the final acceptability degree of the attackers remains the same. Given the simplicity of this proposition, we do not provide a proof for it here.

\begin{proposition}

Given two weighted argumentation frameworks $\AF = \langle \A, \D, w \rangle , \AF' = \langle \A, \D', w \rangle$, and a desirable final acceptability degree function $S: \A \to [0,1]$ such that for all $a \in \A$, $\sigma_{CB}^\AF(a) = S(a)$.
If all the following hold:

\begin{itemize}
    \item $x \in \A$
    \item $Z \subseteq \A \times \{x \}$, and
    \item $\D' = \left(\D \cap (\A \times (\A \setminus \{ x \})) \right) \cup Z$,
    \item $|Z| = \lfloor \frac{w(a)-S(a)}{S(a)} \rfloor$ and $\sum_{(z,x) \in Z} S(z) = -\frac{k(S(x)k+S(x)-w(x))}{S(x)}$ if $w(x) \neq 0$.
\end{itemize}
then for all $a \in \A$, it is the case that $\sigma_{CB}^{\AF'}(a) = S(a)$


\label{prop:solution-CB}
\end{proposition}

Propositions \ref{prop:sol-exp-MB}, \ref{prop:sol-contraction}, and \ref{prop:solution-MB} demonstrate that there are typically a large number of non-minimal argumentation frameworks for the weighted max-based semantics. For the weighted h-categoriser semantics, attacks can be substituted as long as the sum of the final acceptability degrees of the attacking arguments match (see Proposition \ref{prop:solution-HC}), while for weighted cardinality  semantics, attacks can only be substituted by a set of attackers of the same size which final acceptability degree sums to the same value (see Proposition \ref{prop:solution-CB}). We also note that, for all semantics, any argument with an initial weight of 0 can clearly have any number of attacks against it.


\section{Discussion and Future Work}
\label{sec:Discussion}

To this point, we have considered only the complete problem of attack inference in gradual argumentation frameworks, and we now briefly discuss the partial problem. Here, we are given a semantics, a set of arguments, a partial mapping between arguments and initial weights, and a partial mapping between arguments and final acceptability degrees, and we must determine whether attacks (and initial weights which were not provided) can be identified which make the given final acceptability degrees consistent with the semantics.

Since the complete problem is a special case of this partial problem, and since we can still verify correctness in polynomial time, the NP-completeness results of Section \ref{sec:infer_attacks} still apply to the $CB$ and $HC$ semantics. The behaviour of the $MB$ version of the partial problem --- which was polynomial in the complete problem --- is more challenging to characterise. While it is easy to show that there are polynomial instances of the partial problem (e.g., when all initial weights are given and all but one final acceptability degree is provided), we strongly believe that the general case is NP-complete. We leave proof of this result as an element of future work. Other directions for research include providing a mapping from subset-sum to other problems for which solvers exist (e.g., satisfiability) and evaluating the performance of such solvers\footnote{An implementation of our algorithms using an optimised depth-first-search SSP solver can be found at \url{https://github.com/jhudsy/Gradual\_Attack\_Inference}. We do not provide experimental data as the performance of our solver demonstrates the exponential growth of the underlying subset-sum problem. Systems with more than $\sim$13 arguments can only rarely be solved in reasonable time using our implementation.}. Argumentation researchers have found that this approach has often yielded excellent results \cite{thimm_first_2017}. In addition, it may be interesting to evaluate such solvers on different (given) underlying graph topologies. Finally, we focused on three gradual semantics in this paper. However, many other weighted semantics have been proposed (e.g., \cite{TB-sem,gabbay_equilibrium_2015,bonzon_comparative_2016,matt_game-theoretic_2008,leite_social_2011}), and extending the current work to these, and in particular, to probabilistic semantics \cite{li11probabilistic,hunter17probabilistic} would be interesting.

We can consider another inverse problem, namely the computation of semantics for a given weighted argumentation framework and set of final acceptability degrees or preferences over arguments. We note in passing that this problem is trivial to solve in polynomial time, as one simpy needs to check for consistency between the inputs (weighted argumentation framework) and outputs (final acceptability degrees or preferences) for each semantics.

Turning to related work, researchers have examined the notion of \emph{argument realisability} in abstract argumentation frameworks \cite{linsbichlercharacterizing}, seeking to identify an attack relation that yields a specific labelling or extension. Recent work by Mumford et al.~\cite{mumford2022complexity} has shown that --- for complete semantics and IN/OUT labellings --- the problem is NP-complete, whereas it can be solved in polynomial time if UNDEC labellings are also allowed.
The work of Skiba et al. \cite{DBLP:conf/comma/SkibaTRHK22} focuses on whether --- for a given ranking and ranking-based semantics --- we can find an unweighted argumentation framework such that the selected ranking-based semantics induces the ranking when applied to the argumentation framework. They show that the above problem is true for a number of ranking-based semantics, including burden-based and discussion-based semantics \cite{amgoud_ranking-based_2013}, the simple product semantics on social abstract frameworks \cite{leite_social_2011,bonzon_comparative_2016}, and the probabilistic graded semantics \cite{DBLP:conf/comma/ThimmCR18,DBLP:conf/tafa/LiON11}.
Another related area to the current work is \emph{argument synthesis} \cite{niskanen2019synthesizing}, where the attack relation must satisfy some positive and negative constraints, but where the problem is not fully constrained. There are still several open questions regarding the time complexity of these types of problems, even for the case of abstract argumentation semantics \cite{mumford2022complexity}. 

More generally, our work can be seen as a type of inverse argumentation problem. \cite{inverse_problem,oren_analytical} has examined one such inverse problem in the context of the gradual semantics, seeking to identify a set of initial weights given a semantics, a set of arguments, attacks and final acceptability degrees. In the context of abstract argumentation, such inverse problems have examined inferring preferences from justified arguments \cite{mahesar_computing_2018,MahesarOV20}, and has applications in the context of belief revision \cite{baumann2015agm}.

The current paper focuses on the complexity of the underlying decision problem and inferring attacks given the semantics, initial weights, and final acceptability degrees. This --- in a sense --- is unrealistic, and limits the applications of the current work. In most uses of gradual semantics, one considers the preferences obtained from final acceptability degrees rather than the final acceptability degrees themselves, as the latter are not (normally) exposed within an application. Thus, as future work, we intend to consider the problem of inferring attacks from a semantics, initial weights, and preferences over arguments. We believe that the complexity of this problem for a given semantics will be similar to that obtained in the current work. Given this more general problem one could --- for example --- determine whether an agent is rational given the initial weights they assign to a set of arguments, a semantics, and a resultant preference ordering over the arguments. Other applications, as discussed above, include opponent modelling \cite{reinstra13opponent} in the context of dialogue and preference elicitation.

\section{Conclusions}
\label{sec:conclusion}

We have considered the problem of inferring an attack relation given a gradual semantics and all initial weights and final acceptability degrees for an argumentation framework. We have shown that for the weighted max-based semantics, this problem can be solved in linear time, but that the obtained solution is (typically) not unique. For the h-categoriser and cardinality semantics --- where solutions have no redundancies --- the problem becomes NP-complete. Our proofs are based on a reduction to variants of the subset-sum problem, and efficient solvers for this problem can be applied to our work, facilitating its application in domains such as opponent modelling. Finally, we have focused on the complete problem, and there are exciting avenues for future research dealing with its partial form.

\bibliographystyle{unsrt}  
\bibliography{bibliography}

\end{document}